\documentclass{article}
\usepackage{ijcai17}
\usepackage{times}
\usepackage{url}
\usepackage{bm}
\usepackage{amsmath}
\usepackage{amsthm}
\usepackage{amsfonts}
\usepackage{cases}
\usepackage[]{multicol}
\usepackage[]{multirow}
\usepackage{graphicx}
\usepackage{subfigure}

\usepackage{algorithm}
\usepackage{algorithmic}
\usepackage{color}
\usepackage[small]{caption}

\newtheorem{theorem}{Theorem}[]
\newtheorem{proposition}[theorem]{Proposition}
\newtheorem{definition}{Definition}[]

\newcommand{\cduel}{{\sc\textsf{CorrDuel}}}
\newcommand{\cupdate}{{\sc\textsf{CorrUpdate}}}

\title{Correlational Dueling Bandits with Application to Clinical Treatment in Large Decision Spaces}

\author{Yanan Sui, Yisong Yue, Joel W. Burdick \\
  California Institute of Technology \\
  Pasadena, CA 91125 \\
  \{ysui,yyue\}@caltech.edu, jwb@robotics.caltech.edu\\
}
\begin{document}

\maketitle

\begin{abstract}
% We consider sequential decision-making under uncertainty, to optimize an unknown function with feedback as noisy comparisons between different strategies. This problem can be formulated as a $K$-armed Dueling Bandits problem where $K$ is the total number of strategies. Existing dueling bandits algorithms are not efficient when $K$ is very large, even with dependent strategies. This paper studies the dueling bandits problem with large number of dependent arms. The problem is motivated by a clinical decision making process in large decision space. We propose an efficient algorithm \cduel for the problem which makes decisions to simultaneously deliver effective therapy and explore the decision space. Many applications with large volume of parameter selection and sequential decision making could be facilitated by our algorithm. After evaluated the fast convergence of our algorithm in simulation experiments, we applied it on a live clinical trial of therapeutic spinal cord stimulation. It is the first applied algorithm towards spinal cord injury treatments and experimental results show the effectiveness and efficiency of our algorithm.
We consider sequential decision making under uncertainty, where the goal is to optimize over a large decision space using noisy comparative feedback. This problem can be formulated as a $K$-armed Dueling Bandits problem where $K$ is the total number of decisions. When $K$ is very large, existing dueling bandits algorithms suffer huge cumulative regret before converging on the optimal arm. This paper studies the dueling bandits problem with a large number of arms that exhibit a low-dimensional correlation structure. Our problem is motivated by a clinical decision making process in large decision space. We propose an efficient algorithm \cduel which optimizes the exploration/exploitation tradeoff in this large decision space of clinical treatments. 
More broadly, our approach can be applied to other sequential decision problems with large and structured decision spaces.
%Many sequential decision making problems with large and structured decision space could be facilitated by our algorithm. 
We derive regret bounds, and evaluate performance in simulation experiments as well as on a live clinical trial of therapeutic spinal cord stimulation.
%After evaluated the fast convergence of \cduel in analysis and simulation experiments, we applied it on a live clinical trial of therapeutic spinal cord stimulation. 
To our knowledge, this marks the first time an online learning algorithm was applied towards spinal cord injury treatments.  Our experimental results show the effectiveness and efficiency of our approach.
\end{abstract}

% !TEX root = main.tex
\section{Introduction}

In many online learning settings, particularly those that involve human feedback, reliable feedback is often limited to pairwise preferences instead of real valued feedback. Examples include implicit or subjective feedback for information retrieval and recommender systems, such as clicks on search results, and subjective feedback on the quality of recommended care \cite{chapelle2012large,sui2014clinical}.  This setup motivates the dueling bandits problem \cite{yue2009interactively}, which formalizes the problem of online regret minimization via preference feedback (e.g., choosing a pair of arms to be compared at each time step). Many dueling bandits algorithms \cite{yue2009interactively,yue2011beat,zoghi2014relative,ailon2014reducing,komiyama2015regret,wu2016double} have been developed for efficiently computing this problem with independent arms. However, these algorithms are not efficient in situations involving a large number of dependent arms. Specifically, when the time horizon $T$ is smaller than the number of arms $K$, it is hopeless to achieve low regret without leveraging structure among arms. %Moreover, many related problems come equipped with natural correlations between arms.

Our problem is motivated by clinical research for recovering motor function after severe spinal cord injury. Previous research \cite{harkema2011effect} has shown that electrical stimulation applied to the spinal cord via electrode arrays implanted in the epidural space over the lumbosacral area enables paralyzed patients to achieve full weight-bearing standing, improvements in stepping, and partial recovery of lost autonomic functions. Stimulation consists of electrical pulse trains applied to selected electrodes. The challenge is that the optimal stimulus pattern (the choice of active electrodes and their polarities, the pulse amplitude and width, and the pulse train frequency) varies significantly across patients. And even for the same patient, the response to the same stimulus has some variation across trials. Hence, clinicians must determine the optimal stimulus for each patient under noisy conditions, which currently a laborious and ad-hoc approach. %Currently, the search for the optimal stimulating parameters is a laborious and somewhat ad-hoc approach which consumes valuable clinician and patient time, and does not guarantee an optimal outcome.

\begin{figure}[b!]
\centering
\includegraphics[width=0.48\textwidth]{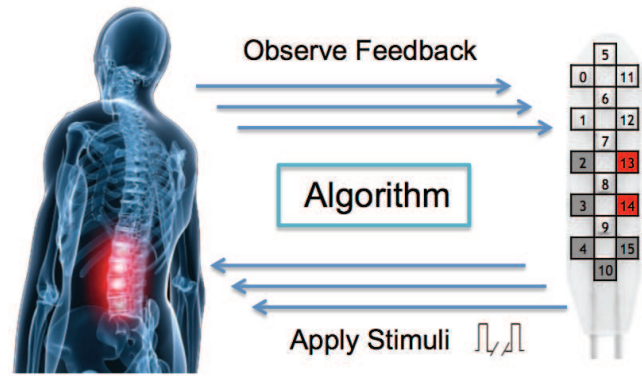}
\caption{The Standing Experiment under spinal stimulation.} 
\label{fig:exp}
\end{figure}

Figure \ref{fig:exp} shows the clinical treatment procedure for {\em stand-training} of paraplegics. During a treatment/optimization session, new stimuli are recommended by the algorithm to be applied to the electrode implanted in the patient. The patient then attempts to stand using the given stimuli, and the observing clinicians compare the patient's standing performance.  The total number of different stimulating configurations is  $\sim 4.3\times 10^7$ due to the complexity of electrodes, and so it is not feasible to search through the whole space. The goal is to develop a algorithm that can automatically select stimuli in order to quickly converge to good treatments.
%Having these noisy comparisons as feedback, we want to continues to explore for the optimal stimulus while also exploiting currently-known good ones. We must spend significant time dwelling upon well-performing stimuli in order to provide the patient with a good therapeutic experience. Since clinical training has a fixed time horizon, we must also maximize total performance during the limited period within which we can search for the optimal solution.

Motivated by this application, we consider the problem of finding optimal stimuli based on the general setting of the multi-armed bandit problem. The classical bandit problem trades off between exploration and exploitation among a number of different arms, each having a quantifiable but stochastic reward with an initially unknown distribution. In contrast, for our clinical problem, the patient's motor response to stimulation is hard to quantify. Neither video motion capture nor electromyographic (EMG) recordings of muscle activity can yet provide a consistent and satisfactory measure of motor skill under stimulation. One reasonably reliable measure is that of pairwise comparisons, e.g., whether one stimulus more effective than another. 
% Good standing performance might map to numerous combinations of muscle activities, and it is not a stationary process. 
While the patient's performance under a specific stimulus is hard to quantify in the clinical setting, we can obtain comparisons of stimuli which are tested within the short time period of one training session.

%The {\em dueling bandit problem} formalizes such online learning problems that use preference feedback instead of absolute rewards. %, and hence it can be used for problems with unquantifiable reward.
%The algorithm we propose in this paper is a variant of the dueling bandit algorithm which is dictated by the clinical demands of our application. In particular, we incorporate the standard dueling bandit algorithm with the dependence of arms that can be captured by some similarity function.

\paragraph{Our Contributions.}
In this paper, we show how to cast the problem of online learning of personalized clinical treatment as a dueling bandits problem with a correlated action space, which we call {\em correlational dueling bandits} .  We present an algorithm which meets the demands of such clinical settings, and can effectively model such correlation dependencies to achieve good performance.
%In this paper, we study a novel class of dueling bandits problems - - as a dueling bandits with large number of dependent arms.
% By adding the structure of correlation, we aim to achieve fast convergence for dueling bandits. 
%Concretely, we propose a novel algorithm \cduel as the incorporation of \cupdate subroutine and Beat-the-Mean algorithm. 
Our algorithm takes advantage of the correlations among different arms to update the whole active set of arms instead of only updating the two dueling arms. This approach achieves fast convergence to the (near) optimal decisions regardless of the large decision space.
We deployed \cduel as the first algorithmic approach to the control of spinal cord stimulation in clinical experiments. We find that \cduel can identify a group of optimal stimuli and help paraplegic human patients to achieve full-weight standing.

% !TEX root = main.tex
\section{Related Work}
\subsection{Multi-Armed Bandits}
The stochastic multi-armed bandits problem \cite{robbins52} refers to an iterative decision making problem in which one repeatedly chooses among K options, such as pulling one of K arms of a bandit machine. In each round, we receive a reward that depends on the arm being selected. Without loss of generality, assume that every reward is bounded between $[0,1]$. The goal then is to minimize the cumulative regret compared to the best arm.

Popular algorithms for the stochastic setting include UCB (upper confidence bound) algorithms \cite{auer2002finite,bubeck2012regret}, and Thompson Sampling \cite{chapelle2011empirical,russo2014learning}.

In the adversarial setting, the rewards are chosen in an adversarial fashion, rather than sampled independently from some underlying distribution.  In this case, regret is rephrased as the difference in the sum of rewards. The predominant algorithm for the adversarial setting is EXP3 \cite{auer2002nonstochastic}. 

% The algorithm we developed in this paper has the potential to incorporate any of the previously mentioned MAB algorithms. Previous work has shown that dueling bandits algorithms enjoy state-of-the-art empirical performance using Sparring \cite{ailon2014reducing}, RUCB \cite{zoghi2014relative}, RMED \cite{komiyama2015regret}, and DTS \cite{wu2016double}.

% \textbf{(YY: This paragraph (and this whole subsection really) sounds out of place.  Did you just copy from UAI paper?)}

\subsection{Correlated Bandits}
The set of candidate actions is very large (or even infinite) in many applications. When that is the case, one must exploit dependencies between payoffs of different decisions in order to arrive at an efficient algorithm.

In some applications, the underlying problem comes equipped with a correlational structure. Various methods of introducing dependence include bandits on trees \cite{uctpaper}, bandits with linear correlations \cite{dani08,abernethy08,abbasi2011improved,gentile2014online} or Lipschitz continuous payoffs \cite{Kleinberg08,Bubeck08}, and Gaussian payoffs \cite{srinivas-icml}. 
% In this paper we pursue a Bayesian approach to bandits, where fine-grained assumptions on the regularity of the reward function can be imposed through proper choice of the prior distribution over the payoff function. 

% \textbf{(YY: this paragraph also sounds off -- the emphasis on Bayesian seems weird.)}

\subsection{Dueling Bandits}
% \textbf{(YY: this is weird too... you need to introduce the dueling bandits problem before you can talk about the stochastic setting. You don't introduce the dueling bandits problem until Section 3.)}
Dueling bandits problem \cite{yue2012k}, as a variant of the multi-armed bandits, takes (noisy) comparative feedback instead of real-valued feedback. It is under the general framework of preference learning (learning with preferential feedback). The dueling bandits problem can also be viewed as a special case of partial monitoring problems \cite{cesa2006regret}. Its problem setting naturally fits in with many applications such as information retrievals and recommender systems. The stochastic dueling bandits problem has been extensively studied in \cite{yue2012k,ailon2014reducing,zoghi2014relative,komiyama2015regret,wu2016double}.

Beyond the stochastic $K$-armed dueling bandits setting, other dueling bandit settings include multi-way preference feedback \cite{sui2014clinical}, continuous-armed convex dueling bandits \cite{yue2009interactively}, contextual dueling bandits which also introduces the von Neumann winner solution concept \cite{dudik2015contextual}, sparse dueling bandits that focus on the Borda winner solution concept \cite{jamieson2015sparse}, Copeland dueling bandits that focus on the Copeland winner solution concept \cite{zoghi2015copeland}, and adversarial dueling bandits \cite{gajane2015relative}. It would be interesting to study how to extend our analysis to these other settings as well.

% The dueling bandits problem can also be viewed as a special case of partial monitoring problems \cite{cesa2006regret}.  In partial monitoring, the feedback received is assumed to be only indirectly related to the actual rewards.  However, generic algorithms for partial monitoring problems are generally not competitive compared to algorithms specifically designed for the dueling bandits problem.

% !TEX root = main.tex
\section{Problem Statement}
In the classical dueling bandits problem, at each iteration $t$, the following happens:
\begin{itemize}
\item The algorithm chooses a pair of actions $b^{(1)}(t)$ and $b^{(2)}(t)$ from a set of $K$ possible actions.
\item The algorithm duels $b^{(1)}(t)$ and $b^{(2)}(t)$ and receives (noisy) feedback corresponding to the winner. 
\end{itemize}
%The classical dueling bandit problem receives feedback in the form of a comparison between a pair of arms in each test. 

Our procedure can be described as follows.  There is a set of arms $\mathcal{B} = \{b_1,\cdots,b_K\}$, and a total number of $T$ tests to be performed. At each time step, a pair of arms are chosen from the set $\mathcal{B}$ and a (noisy) comparison of them is observed. $T$ is determined before we run the algorithm. The set of arms are correlated and $T \le |\mathcal{B}| = K$ in general.

We follow the original notation of the dueling bandit problem. For two arms $b_i$ and $b_j$ sampled from $\mathcal{B}$, we write the comparison factor as
    \[ \epsilon(b_i,b_j) = P(b_i\succ b_j)-1/2 ,\]
where $P(b_i\succ b_j)$ is the probability that $b_i$ dominates $b_j$ and $\epsilon(b_i,b_j)\in [-1/2,1/2]$ represents the priority between $b_i$ and $b_j$. We define $b_i\succ b_j \Leftrightarrow \epsilon(b_i,b_j)>0$. We use the notation $\epsilon_{i,j} \equiv \epsilon(b_i,b_j)$ for convenience. Note that $\epsilon(b_i,b_j) = -\epsilon(b_j,b_i)$ and $\epsilon(b_i,b_i)=0$. We assume the distribution of reward for each arm is stationary so that all comparison factors converge in [-1/2,1/2]. We also assume $w.l.o.g.$ that the arms are indexed in preferential order $b_1\succ b_2\succ \cdots \succ b_K$ so that there is one preferred arm.

Our goal is to minimize the total regret:
\[ R_T=\sum_{t=1}^T \epsilon(b_1,b^{(1)}(t)) + \epsilon(b_1,b^{(2)}(t))\]
The total regret $R_T=0$ if we constantly choose $b(t)=b_1$ during the experiment. $R_T=\Theta (T)$ is linear $w.r.t. \ T$ if we constantly choose $b(t) \in \mathcal{B}$.

% Our problem of correlational dueling bandits takes the correlations among arms into consideration. For any pair of arms $b_i$ and $b_j$, we consider the dependence between them can be captured by some similarity function $r_{ij} \in [0,1]$. And it satisfies:
% \begin{itemize}
% \item $r_{ij} = r_{ji}$
% \item $r_{ij} = 0 \iff b_i$ and $b_j$ are not correlated.
% \item $r_{ij} = 1 \iff b_i = b_j$.
% \end{itemize}

We also inherit two properties of the comparison factors from the original dueling bandit problem:

{\bf Strong Stochastic Transitivity.} For any triplet of arms $b_i \succ b_j \succ b_k$, we assume $\epsilon_{i,k} \ge \max \{ \epsilon_{i,j}, \epsilon_{j,k}\}$.

{\bf Stochastic Triangle Inequality.} For any triplet of arms $b_i\succ b_j \succ b_k$, we assume $\epsilon_{i,k} \le \epsilon_{i,j}+\epsilon_{j,k}$. This can be viewed as a diminishing returns property.

% \textbf{(YY: you're repeatedly talking about the fact that you're introducing correlations to DB.  You should just pick one place to do it.)} 

\textbf{Correlational Dueling Bandits.} When the size of the decision set, $K$, is large, it is unavoidable to carry out a very large number of tests before the algorithm converges to its optimal solution. In some applications like our clinical example, each test is expensive and time-consuming. The number of tests -- the time horizon of an algorithm -- is often predetermined by clinical conditions. 
%It is infeasible to apply the original dueling bandit algorithms to these applications due to the large decision space.
We thus augment the dueling bandits problem into correlational dueling bandits, which takes the correlations among arms into consideration. 
% We assume there exists an underlying $smooth$ function $f(b_i,b_j): \mathcal{B} \times \mathcal{B} \rightarrow [-1/2,1/2]$ where $f(b_i,b_j) = \epsilon (b_i \succ b_j)$. 
For any pair of arms $(b_i,b_j) \in \mathcal{B}^2$, we consider the dependence between them are captured by some similarity function $r_{ij} \in [0,1]$, and it satisfies:
\begin{itemize}
\item $r_{ij} = r_{ji}$;
\item $r_{ij} = 0 \iff b_i$ and $b_j$ are not correlated;
\item $r_{ij} = 1 \iff b_i = b_j$.
\end{itemize}

For all tuples $(b_i;b_j,b_k) \in \mathcal{B}^3$, if we play  pair $(b_i,b_j)$ once and observe $b_i \succ b_j$, we define $\kappa(b_k;b_i,b_j)$ to be the update of wins of $b_k$ and $\tau(b_k;b_i,b_j)$ to be the update of plays of $b_k$. 
$\kappa(\cdot;\cdot,\cdot)$ and $\tau(\cdot;\cdot,\cdot)$ represent the dependent structure of the tuple arms. They could be functions of $r_{ij}$, $r_{ik}$, and $r_{jk}$. 
% In the following context, we assume this update is unbiased.

% The standard notion of correlation coefficient, $r_{XY}=E[XY-E[X]E[Y]]/\sqrt{Var[X]Var[Y]}$, is used in our experiments. However, one can use any measure as a basis for $r_{XY}$ as long as the above three properties are satisfied.

% \textbf{(YY: too informal)}
In our synthetic experiments, we assume the input space (set of arms) $\mathcal{B}$ has dependent structure and there exists an underlying utility function $f(b): \mathcal{B} \rightarrow \mathbb{R}$ over input space which we cannot observe directly. Our observations are the noisy comparisons between pairs of arms (e.g., $b_i$ and $b_j$) which can be viewed as the noisy comparison of utility values (e.g., $f(b_i)$ and $f(b_j)$). The properties of strong stochastic transitivity, stochastic triangle inequality, and the dependency assumptions on $r_{ij}$ generally hold for a wide range of applications. In the clinical experiments, we extract comparisons from physician's online judgment.

% !TEX root = main.tex
\section{Algorithm}
Our algorithm, \cduel~as shown in Algorithm~\ref{alg:alg1}, is a correlational dueling bandits algorithm based on the Beat-the-Mean algorithm \cite{yue2011beat}. It uses observational feedback and the correlational structure to successively remove suboptimal arms, while keeping the optimal one(s) in the sample space with high probability. The inputs to \cduel~are the set of arms $\mathcal{B}$, the total number of iterations $T$, and the correlational structure $(\kappa,\tau)$. 
% \textbf{(YY: no mention of Algorithm 1?)}

{\em Parameters-Initialization} (Algorithm \ref{alg:alg2}) defines the set of active arms $W_{\ell}$, whose size shrinks as more tests are completed. For each arm $b$, let $n_b$ be the total number of comparisons between $b$ and other arms, and let $w_b$ be the total number of wins against all other arms. Let
$\hat{P}_b$ be the empirical average of $P(b\succ b')$ for all $b'$ in $W_{\ell}$, and let $\hat{P}_{b,n}$ be the value of $\hat{P}_b$ after $n$ comparisons between arm $b$ and any other arms.  Set the confidence interval of $P(b\succ b')$ as: 
    \[ \hat{C}_{b,n}=(\hat{P}_{b,n}-c_{\delta}(n), \hat{P}_{b,n}+c_{\delta}(n)) \]
where $c_{\delta}(n)=\sqrt{(1/n) log(1/\delta)}$, and $\delta$ is the confidence that $P(b\succ b')$ lies in $\hat{C}_{b,n}$.  The function $c_{\delta}(n)$ decreases as the number of comparisons $n$ increases. By properly setting parameter $\delta$, the optimal reward can be reached within the fixed time horizon as shown in Proposition~\ref{prop:1}.

{\em Active-Elimination} (Algorithm \ref{alg:alg3}) is the key part of \cduel. For each pair of tests, two arms are randomly chosen from $W_{\ell}$. The randomized selection method enjoys low-variance total regret in general. For each arm $b$, the values of $w_b$, $n_b$ and $\hat{P_b}$ are updated, as is the corresponding confidence radius $c^*$.  An arm $b$ dominates another arm $b'$, if their confidence intervals do not overlap, and the inferior arm is eliminated from $W_{\ell}$. The algorithm runs until the time horizon $T$ is reached, or only one active arm remains.

\cupdate~(Algorithm~\ref{alg:alg4}) is the subroutine of {\em Active-Elimination} (Algorithm \ref{alg:alg3}) which updates the weights of $b_k$ by rules $\kappa(\cdot;\cdot,\cdot)$ and $\tau(\cdot;\cdot,\cdot)$.
% We call the updating rules for $w_b$, $n_b$ shown in line 6-11 the \cupdate.
In the classical dueling bandits setting, we assume arms are independent. For independent arms, if we have one comparison between $b_i$ and $b_j$ and gets $b_i \succ b_j$, 
we only update the weights for arm $b_i$ and $b_j$:
\begin{equation}
w_i \leftarrow w_i + 1, n_i \leftarrow n_i + 1    \label{eqn:1}
\end{equation}
\begin{equation}
w_j \leftarrow w_j, n_j \leftarrow n_j + 1    \label{eqn:2}
\end{equation}

For a large decision space, existing dueling bandits algorithms are extremely slow if one does not exploit dependencies among arms, even if they can achieve provably optimal cumulative regret (w.r.t. independent arms). When the arms are correlated and the correlation between any pair of arms $b_i$ and $b_j$ is measured properly by $r_{ij}$, we can update all active arms at each iteration.

% The standard notion of correlation coefficient, $r_{XY}=E[XY-E[X]E[Y]]/\sqrt{Var[X]Var[Y]}$, is used in our experiments. However, one can use any measure as a basis for $r_{XY}$ as long as $r_{XY}\in[0,1]$, $r_{XY}=1$ when $X=Y$, and $r_{XY}=0$ when $X$ has an ''irrelevant'' relation to $Y$. The coefficient $r$ can take negative values, but the algorithm doesn't use negative values to update Equations (5)(6).\\
% The motivation behind Equations (3)(4), which was explained in the paragraph below the equations, has been expanded in the current draft.  It is a compromise between clinical constraints and an explicit Bayesian update. 

As shown in Algorithm~\ref{alg:alg4}, we update every arm $b_k$ after comparing arms $b_i$ and $b_j$ ($w.l.o.g.$ assume $b_i \succ b_j$) via:
\begin{equation}
w_k \leftarrow w_k + \kappa(b_k;b_i,b_j)
\end{equation}
\begin{equation}
n_k \leftarrow n_k + \tau(b_k;b_i,b_j)
\end{equation}
where $\kappa(\cdot;\cdot,\cdot)$ and $\tau(\cdot;\cdot,\cdot)$ represent the correlational structure, which is assumed to satisfy:

\begin{itemize}
\item $0 \le \kappa(b_k;b_i,b_j) \le \tau(b_k;b_i,b_j) \le 1$;
\item if $b_k = b_i$, $\kappa(b_k;b_i,b_j) = \tau(b_k;b_i,b_j) = 1$;
\item if $b_k = b_j$, $\kappa(b_k;b_i,b_j) = 0$, $\tau(b_k;b_i,b_j) = 1$.
\end{itemize}
These updates are based on the assumption that $\kappa(\cdot;\cdot,\cdot)$ $\tau(\cdot;\cdot,\cdot)$ is an unbiased estimation of the dependent structure. The \cupdate~subroutine (Algorithm~\ref{alg:alg4}) can efficiently update all arms at each iteration. So \cduel~enjoys fast convergence towards the near optimal arms.

\begin{algorithm}[t]
   \caption{\cduel}
   \label{alg:alg1}
\begin{algorithmic}[1]
   \STATE {\bfseries Input:} $\mathcal{B}$, $T$, $(\kappa,\tau)$
   \STATE {\bfseries Input:} $c_{\delta}(n)=\sqrt{(1/n)log(1/\delta)}$ 
   \STATE {\bfseries Run:} [Parameters-Initialization]
   \STATE {\bfseries Run:} [Active-Elimination]
   \STATE {\bfseries return} $b^* \ \ \ \ //\ Optimal\ arm$
\end{algorithmic}
\end{algorithm}
\begin{algorithm}[t]
   \caption{Parameters-Initialization}
   \label{alg:alg2}
\begin{algorithmic}[1]
%    \STATE {\bfseries Input:} $\{b_1,...,b_K\}$, $T$
%    \STATE {\bfseries Input:} $c_{\delta}(n)=\sqrt{(1/n)log(1/\delta)}$ 
   \STATE $W_1\gets \mathcal{B} \ \ \ \ // \ set\  of\  active\  arms$
   \STATE $\ell \gets 1 \ \ \ \ // \ rounds$
   \STATE $\forall b \in W_{\ell}, \ n_b \gets 0 \ \ \ \ // \ comparisons$
   \STATE $\forall b \in W_{\ell}, \ w_b \gets 0 \ \ \ \ //\ priorities$
   \STATE $\forall b \in W_{\ell}, \ \hat{P}_b\equiv w_b/n_b$, or 1/2 if $n_b=0$
   \STATE $n^* \equiv min_{b\in W_{\ell}} n_b$
   \STATE $c^* \equiv c_{\delta}(n^*)$, or  1 if $n^* = 0 \ \ \ \  // \ confidence\ radius$
   \STATE $t \gets 0 \ \ \ \ //\ total\ number\ of\ iterations$
   \STATE {\bfseries return} all new parameters
\end{algorithmic}
\end{algorithm}
\begin{algorithm}[t]
   \caption{Active-Elimination}
   \label{alg:alg3}
\begin{algorithmic}[1]
%    \STATE {\bfseries Input:} $\{b_1,...,b_K\}$, $T$, $r(\cdot, \cdot)$
%    \STATE {\bfseries Input:} parameters generated in [Parameters-Initialization]
   \WHILE {$|W_{\ell}|>1$ and $t\leq T$}
	\STATE select $b_i, b_j \in W_{\ell}$ at random
	\STATE compare selected arms (assume $b_i \succ b_j$)
	    \FOR {all $b_k \in W_{\ell}$}
			\STATE update $w_k$, $n_k$ by \cupdate
	    \ENDFOR \label{lin:update2}
		\IF {$\min_{b'\in W_{\ell}} \hat{P}_{b'}+c^* \le \max_{b\in W_{\ell}} \hat{P}_b-c^*$}	\label{lin:elim}
			\STATE $b' \gets \arg\min_{b\in W_{\ell}} \hat{P}_b$
			\STATE $\forall b\in W_{\ell}$, delete comparisons with $b'$ from $w_b$, $n_b$
			\STATE $W_{\ell+1} \gets W_{\ell} \backslash \{b'\} \ \ \ \ //\ update\ working\ set$
			\STATE $\ell \gets \ell+1 \ \ \ \ //\ new \  round$
		\ENDIF
   \ENDWHILE 
   \STATE {\bfseries return} {$b^* = \arg \max_{b\in W_{\ell}} \hat{P}_b$}
\end{algorithmic}
\end{algorithm}

\begin{algorithm}[t]
	\caption{\cupdate}
    \label{alg:alg4}
\begin{algorithmic}[1]
	\STATE {\bfseries Input:} $b_k, b_i \succ b_j$
    \STATE $w_k \leftarrow w_k + \kappa(b_k;b_i,b_j)$
	\STATE $n_k \leftarrow n_k + \tau(b_k;b_i,b_j)$
    \STATE {\bfseries return} $w_k$, $n_k$
\end{algorithmic}
\end{algorithm}

\begin{definition}
	\label{def:1}
    $\varepsilon$-optimal arm. If arm $b$ satisfies $\epsilon(b_1,b) \le \varepsilon$, then $b$ is an $\varepsilon$-optimal arm.
\end{definition}

\begin{proposition}
	\label{prop:1}
If~$\exists \mu > 0$ such that $\tau(b_k;b_i,b_j) \ge \mu$ for every  $(b_i,b_j,b_k) \in \mathcal{B}^3$, then with high probability, the cumulative time to achieve purely $\varepsilon$-optimal arms $T(\varepsilon)$ is bounded by:
$$T(\varepsilon) = \mathcal{O}\left( \frac{1}{\mu \varepsilon^2} \log \frac{1}{\delta} \right) .$$
\end{proposition}

\begin{proof}
Proposition~\ref{prop:1} holds based on the Theorem 1 of \cite{yue2011beat}.
After $t$ iterations, since $\tau(b_k;b_i,b_j) \ge \mu$, we have $n^* \ge \mu t$. Then $c^* = c_{\delta}(n^*)=\sqrt{(1/n^*)log(1/\delta)} \le \sqrt{(1/\mu t)log(1/\delta)}$. Notice $c^*$ is a function of time step $t$.

For any arm $b$ which is not $\varepsilon$-optimal (satisfies $\epsilon(b_1,b) > \varepsilon$), with probability $1-\delta$, $\hat{P}_{b_1} - \hat{P}_b > \varepsilon C_{\delta}$ holds for some fixed concentration parameter $C_{\delta}$. Suppose arm $b$ has not been eliminated at iteration $t$. Then from elimination criterion Line~\ref{lin:elim} of Algorithm~\ref{alg:alg3} we have $\varepsilon C_{\delta} < \hat{P}_{b_1} - \hat{P}_b < 2c^* \le 2\sqrt{(1/\mu t)log(1/\delta)}$. The inequality breaks when $t \ge \frac{4}{\mu \varepsilon^2 C_{\delta}^2} \log \frac{1}{\delta} = \mathcal{O}\left( \frac{1}{\mu \varepsilon^2} \log \frac{1}{\delta} \right)$.
\end{proof}

Notice, the iteration time $T(\varepsilon)$ in Propositions~\ref{prop:1} does not depend on $|\mathcal{B}|=K$, which suggests the fast convergence of \cduel~in large decision spaces.

\begin{figure}[h]
\centering
\includegraphics[width=0.5\textwidth]{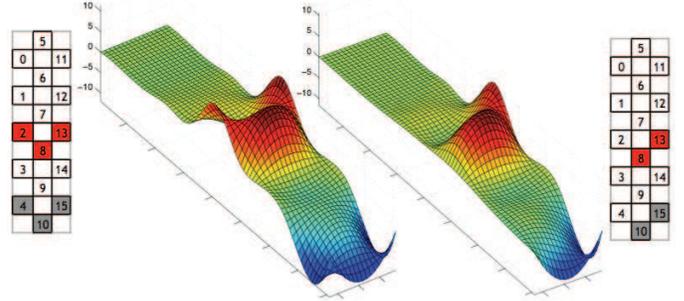}
\caption{Depiction of the correlations between two different multi-electrode stimulating configurations.}
\label{fig:corr}
\end{figure}

In our application of \cduel~to selection of optimal multi-electrode stimulating parameters for paraplegic, we define the similarity of different configurations to be the correlation coefficient of electrical potential fields generated by the two different electrode stimulating configurations. Since the correlation coefficient function $r(\cdot, \cdot)$ has support on $[-1,1]$, we only update with the \cupdate~rule when $r(\cdot, \cdot) \ge 0$. The existence of negative $r$ values is based on clinical observations. The correlational property arises from analysis of electric fields applied by the array as shown in Figure~\ref{fig:corr}. 

The standard notion of correlation coefficient, $r_{XY}=E[XY-E[X]E[Y]]/\sqrt{Var[X]Var[Y]}$, is used in our experiments. However, one can use any measure as a basis for $r_{XY}$ as long as $r_{XY}\in[0,1]$, $r_{XY}=1$ when $X=Y$, and $r_{XY}=0$ when $X$ has an ``irrelevant'' relation to $Y$. The coefficient $r$ can take negative values, but the algorithm doesn't use negative values  for its updates.

% The demonstration of similarity comparison is shown in Figure \ref{fig:corr}. 
% We also use clinical knowledge to restrict the input space from around $4.3 \times 10^7$ to be on the order of $10^3$. And it is still a very large input space considering the number of trials, or arm pulls are on the order of $10^2$.

% In our clinical application the correlational property arises from analysis of electric fields applied by the array as shown in Figure~\ref{fig:corr}.
For correlated arms, we perform an update for every arm $k$ for which $r_{ik}, r_{jk} > 0$ as follows:
\begin{equation}
\kappa(b_k;b_i,b_j) \leftarrow \frac{\log r_{jk}}{\log r_{ik} + \log r_{jk}} \cdot \frac{r_{ik} + r_{jk}}{1 + r_{ij}} \label{eqn:update3}
\end{equation}
\begin{equation}
\tau(b_k;b_i,b_j) \leftarrow \frac{r_{ik} + r_{jk}}{1 + r_{ij}}
\end{equation}

\begin{proposition}
	\label{prop:2}
If $\exists \mu > 0$ such that $r_{ij} \ge \mu$ for every pair $(b_i,b_j) \in \mathcal{B}^2$, then with high probability, the cumulative time to achieve purely $\varepsilon$-optimal arms $T(\varepsilon)$ satisfies:
$$T(\varepsilon) = \mathcal{O}\left( \frac{1}{\mu \varepsilon^2} \log \frac{1}{\delta} \right) .$$
\end{proposition}

\begin{proof}
If $r_{ij} \ge \mu$ for every pair $(b_i,b_j) \in \mathcal{B}^2$, since $r_{ij} \le 1$, $\tau(b_k;b_i,b_j) = \frac{r_{ik} + r_{jk}}{1 + r_{ij}} \ge \frac{2\mu}{2} = \mu$ for every tuple $(b_i,b_j,b_k)$. The result follows from substituting it into Proposition~\ref{prop:1}.
\end{proof}

The \cupdate~subroutine above updates the dueling pair $b_i, b_j$ in the same way as if they are independent since (5) and (6) will collapse to (1) and (2) for $b_i$ and $b_j$.
For extreme cases, if $b_i \succ b_j$ and arm $b_k$ is very close to $b_j$. We have $r_{ik} \simeq 1$ and $r_{jk} \simeq r_{ij}$, the updating rules for arm $b_k$ will be close to the updates of arm $b_j$. If $b_k$ is far from both $b_i$ and $b_j$, (5) and (6) guarantees that the update for $b_k$ is very small since we acquire little information about $b_k$ from far away comparisons. Also, if $b_i$ and $b_j$ are less dependent (with smaller $r_{ij}$), we would expect to acquire larger updates for the points in between. 
% The typical correlations among arms $b_i, b_j, b_k$ are shown in Figure \ref{fig:demo}. Long dashed lines represent smaller $r(\cdot, \cdot)$ values. When $b_k$ is close to $b_i$ or $b_j$, it gains more updates comparing to far away arms.

One can also consider a Bayesian version, e.g., by using Gaussian processes. In this paper, we focus on a frequentist approach, which is a better model of the clinical application.

% \begin{figure}[h]
% \centering
% \includegraphics[width=0.5\textwidth]{figures/demo.pdf}
% \caption{Demonstration of how \cdupdate works.}
% \label{fig:demo}
% \end{figure}

% !TEX root = main.tex
\section{Experiments}
We evaluated our approach in two settings, synthetic simulations and a real clinical application of online optimization for spinal cord stimulation therapy.
%We first evaluated our algorithm on synthetic data. We also applied the algorithm to a real clinical application: online optimization for spinal cord stimulation therapy. 
In our controlled synthetic experiments, we seek to address the following questions: 
\begin{itemize}
\item How does the algorithm compare against standard dueling bandit algorithms? 
\item How effective is it in terms of convergence? 
\end{itemize}
We compare the algorithm against Beat-the-Mean, RUCB,  and Sparring algorithm with UCB1. These three algorithms are the representative dueling bandits algorithms designed for independent arms, which do not, however, leverage the correlations between arms.

\begin{figure}[t]
\centering
\includegraphics[width=0.35\textwidth]{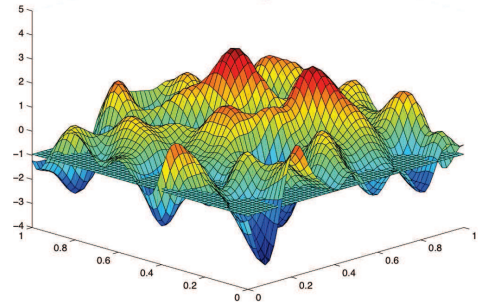}
\caption{Mean function sampled from a Gaussian process.}
\label{fig:syn1}
\end{figure}

\subsection{Simulation Experiments}

\textbf{Setup.} We first evaluate the algorithm with simulation experiments. The purpose of this experiment is to validate our algorithm, and demonstrate its quick convergence when the arms are dependent. 
To generate the underlying utility function over correlated arms, we sampled random functions from a zero-mean Gaussian Process with squared exponential kernel over the sample space $\mathcal{B} = [0,1] \times [0,1]$, uniformly discretized into $50 \times 50$ points (set of arms) and used this function as the mean function for the $2500$ arms. We chose $\sigma = 0.5$ as the standard deviation of arms. One evaluation of the mean functions is shown in Figure \ref{fig:syn1}. The utility function is not necessarily convex or simple. Within each iteration, we sample 2 points in the active set and compare their (noisy) sampling values to get the $\{0,1\}$ feedback of the duel. We run the duel for $T = 100$ iterations for 10000 trials for each of the 4 comparing algorithms.

\textbf{Results.} We report a notion of {\em regret} as the stepwise regret instead of the cumulative regret. It converges to zero as iteration number goes to infinity for every no-regret algorithm. As seen in Figure \ref{fig:syn2}, \cduel~converges much faster than the other three algorithms since it takes the advantage of the dependent arms. The independent-armed dueling bandits algorithms require an exhaustive searching period which is significantly larger than the time horizon we use here before concentrating on the (near) optimal arms.

\begin{figure}[t]
\centering
\includegraphics[width=0.45\textwidth]{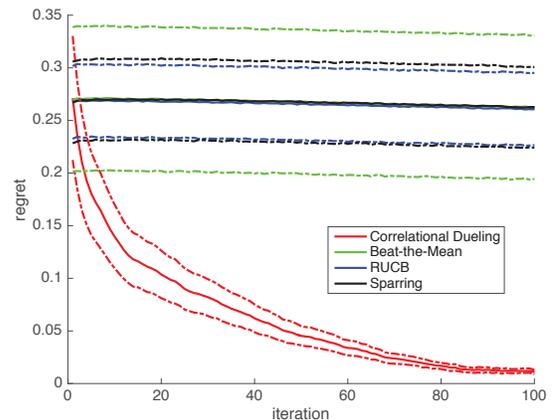}
\caption{Regret versus iteration. The dashed lines represent one standard deviation.}
\label{fig:syn2}
\end{figure}

\subsection{Human Experiments}

\textbf{Background.} As depicted in  Figure~\ref{fig:exp} from before, our human clinical experiments involve optimizing a system for stand training under spinal cord stimulation with spinal cord injury patients.  The subject practices standing under spinal stimulation using a stand frame for assistant in balance. The training processes largely follow the procedures in \cite{rejc2015effects}. Two trainers on the subject's left and right protect and assist the subject. Within each experiment, a specific stimulating pattern (a combination of active electrode selections, the polarity of the actively selected electrodes, and the stimulation amplitude and frequency) is applied through the implanted electrode array and its controlling circuitry. An anonymous short video\footnote{\url{https://youtu.be/loJLtbcUBDM}} shows the standing quality under different stimuli. The first part shows a low quality bipedal standing and the second part shows a better standing, both with electrical spinal cord stimulation. Different standings could look similar for the non-specialist.

The participants are under stable medical condition and have no musculoskeletal dysfunction that might interfere with stand training. They have no motor response present in leg muscles during transcranial magnetic stimulation, indicating that there are no strongly active neural pathways connecting cortex and lower limb muscles. No volitional control can be achieved during voluntary movement attempts in leg muscles as measured by EMG activity. 

% We define the similarity of different configurations to be the correlation coefficient of electrical potential fields generated by the two different configurations. Since the correlation coefficient function $r(\cdot, \cdot)$ has support on $[-1,1]$, we only update with the \cupdate rule when $r(\cdot, \cdot) \ge 0$. The existence of negative $r$ values is based on clinical observations. The demonstration of similarity comparison is shown in Figure \ref{fig:corr}. 
\textbf{Setup.} We  use clinical knowledge to restrict the decision space from around $4.3 \times 10^7$ to be on the order of $10^3 \sim 10^4$. It is still a very large decision space considering the number of trials, or arm pulls, are on the order of $10^2$.

% In our clinical application the correlational property arises from analysis of electric fields applied by the array.
% For correlated arms, we perform an update for every arm $k$ for which $r_{ik}, r_{jk} > 0$ as follows:
% \begin{equation}
% \kappa(b_k;b_i,b_j) \leftarrow \frac{\log r_{jk}}{\log r_{ik} + \log r_{jk}} \cdot \frac{r_{ik} + r_{jk}}{1 + r_{ij}} \label{eqn:update3}
% \end{equation}
% \begin{equation}
% \tau(b_k;b_i,b_j) \leftarrow \frac{r_{ik} + r_{jk}}{1 + r_{ij}}
% \end{equation}

% \begin{figure}[h]
% \centering
% \includegraphics[width=0.5\textwidth]{figures/correlation.pdf}
% \caption{Capture the correlations between different stimulating configurations}
% \label{fig:corr}
% \end{figure}

A total of 414 experimental comparisons were done with two patients under the \cduel~algorithm. Each trial lasted for about 5 minutes. Within each trial, one stimulating pattern was generated by the 16-channel electrode. The patterns were unchanged within each trial. For a fixed electrode configuration, the stimulation frequency and amplitude were modulated synergistically in order to find the best values for effective weight-bearing standing. We optimized the electrode patterns with \cduel~and performed exhaustive search for stimulation frequency and amplitude over a narrow range. 
% Different stimulating patterns are exploited along the trials in order to find the most effective ones. Specific electrode configuration adjustments were defined to seek improvements of different aspects of motor output. The guideline for the parameter-tuning is related to both previous findings reported in the literature and results of previous experiments performed on the same research participant.

% The patient performed experimental and training sessions for standing using a custom designed standing frame comprised of horizontal bars anterior and lateral to the individual. These bars were used for upper extremity support and balance assistance as needed. If the knees or hips flexed beyond a safe standing posture, external assistance was provided at the knees to promote extension, and at the hips to promote hip extension and anterior tilt. Facilitation was provided either manually by a trainer or by elastic bungee cords, which were attached between the two vertical bars of the standing apparatus. Mirrors were placed in front of the participant and laterally to him, in order to allow a better perception of the body position via visual feedback, conditioned on the lack of proprioceptive sensory feedback.

Stimulation began while the patient was seated. Then the participant initiated the sit to stand transition by positioning his feet shoulder width apart and shifting his weight forward to begin loading the legs. 
% As shown in Fig. \ref{fig:exp}, the participant used the horizontal bars of the standing apparatus during the transition phase to balance and to partially pull himself into a standing position. Trainers positioned at the pelvis and knees manually assisted as needed during the sit to stand transition.

\textbf{Results.} For the clinical experiments, we cannot create a direct plot for regrets since the ground truth optimal stimulation is unknown. In the experiments, we observed the convergence of \cduel, which is not possible for independent-armed dueling bandits algorithms. The set of (near) optimal configurations found by \cduel~is shown in Figure \ref{fig:human}. We compared the performance of \cduel~to the optimal selections found heuristically for each patient by clinicians, which are shown in Figure \ref{fig:human2}. We found that the manual selection is a subset of the algorithm's selection, and there exist high performing configurations (e.g., the 2nd in Figure \ref{fig:human}) found by the algorithm which are not in the manual selection. This shows that \cduel~is performing no worse than specialized physicians.

\begin{figure}[t]
\centering
\includegraphics[width=0.4\textwidth]{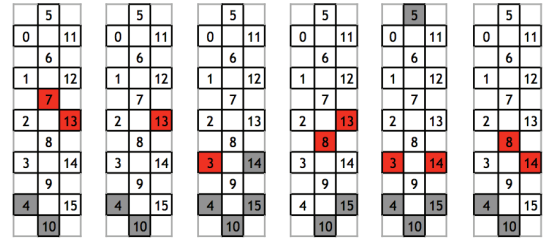}
\caption{The set of (near) optimal configurations found by the algorithm for a specific patient (in decreasing order in terms of performances).}
\label{fig:human}
\end{figure}

\begin{figure}[t]
\centering
\includegraphics[width=0.2\textwidth]{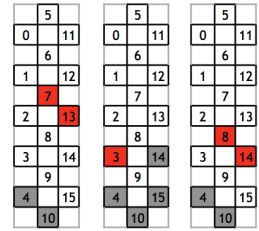}
\caption{The set of (near) optimal configurations found by physician's manual pick for that specific patient (in decreasing order in terms of performances).}
\label{fig:human2}
\end{figure}

% !TEX root = main.tex
\section{Conclusion and Discussion}

Our analysis and simulation demonstrate that \cduel~indeed exhibits fast convergence properties compared to independent-armed dueling bandits algorithms when correlation information is available. We deployed this algorithm in clinical experiments for the control of spinal cord stimulation and showed that \cduel~performs no worse than specialized physicians. We believe that our result provides an important step towards employing machine learning algorithms in many problems with a large volume of parameter selection and sequential decision making. These problems could be facilitated by our algorithm, which simultaneously delivers effective decisions and explores the decision space based on comparative feedback.

The \cupdate~subroutine is easy to incorporate with Beat-the-Mean algorithm to achieve efficient \cduel. Although we developed \cduel~specifically based on Beat-the-Mean, \cupdate~is a more general approach which has potential to incorporate with the existing dueling bandits algorithms. For instance, it can incorporate with RUCB to realize a variant of RUCB for dependent arms by updating the wins $w_{ij}$ with \cupdate.

To our knowledge, our work is the first to apply an algorithmic approach towards spinal cord injury treatments. The algorithm could find a proper set of optimal stimulating configurations within the test time horizon. 
We achieved good performance in both simulations and human experiments. The paraplegic human patients could achieve full-weight standing under the stimulation provided by our algorithm.

% Moving forward, it would be interesting to provide more efficient correlational dueling algorithms .

\section*{Acknowledgements}
This research was supported in part by Caltech/JPL PDF IAMS100224, NIH-U01-EB007615-08, NIH-U01-EB015521-05, and a gift from Northrop Grumman.

%% The file named.bst is a bibliography style file for BibTeX 0.99c
%\newpage
\small
\bibliographystyle{named}
\bibliography{reference}

\end{document}